\documentclass[11pt]{article}
\usepackage{amssymb, amsthm,amsmath, bm, bbm, mathrsfs, graphicx, color}
\usepackage{setspace}
\usepackage{kpfonts}
\usepackage{multicol}
\usepackage{natbib}
\usepackage{perpage} 

\newtheorem{proposition}{Proposition}

\begin{document}
\onehalfspacing
\begin{center}

\textbf{\sc \Large  Clustering categorical data via ensembling \\dissimilarity matrices}

\vspace{0.3cm}

{Saeid Amiri\footnote{Corresponding author: saeid.amiri1@gmail.com},  ~~Bertrand Clarke~ and~~Jennifer Clarke  }

\vspace{0.3cm}

{\it  Department of Statistics, University of Nebraska-Lincoln, Lincoln, Nebraska, USA}

\end{center}

\begin{abstract}
We present a technique for clustering categorical data by generating many dissimilarity matrices
and averaging over them. We begin by demonstrating our technique on low dimensional categorical data 
and comparing it to several other techniques that have been proposed.  Then we give
conditions under which our method should yield good results in general.  Our method extends to
high dimensional categorical data of equal lengths by ensembling over many choices of explanatory variables.
In this context we compare our method with two other methods.   Finally, we extend 
our method to high dimensional
categorical data vectors of unequal length by using alignment techniques to equalize
the lengths.  We give examples to show that our method continues to provide good
results, in particular, better in the context of genome sequences than clusterings suggested by phylogenetic trees.
\end{abstract}


\section{Introduction}
Clustering is a widely used unsupervised technique for identifying natural classes
within a set of data.   The idea is to group unlabeled data into subsets 
so the within-group homogeneity is relatively high and the between-group 
heterogeneity is also relatively high.   The implication is that the groups should
reflect the underlying structure of the data generator (DG).  Clustering of 
continuous data has been extensively studied over many decades leading to 
several conceptually disjoint categories of techniques.  However,
the same cannot be said for the clustering of categorical data which,
by contrast, has not been developed as extensively.

We recall that categorical data are discrete valued measurements.  Here, we will assume 
there are finitely many discrete values and that there is
no meaningful ordering on the values or distance between them  (i.e., nominal).  
For instance, if the goal is to cluster genomes, the data consist of
strings of four nucleotides (A, T, C, G).  Typically these strings are high dimensional
and have variable length.
There are many other similar examples such as clustering a population
based on the presence/absence of biomarkers in specific locations.

We first propose a technique for low dimensional equi-length categorical data vectors 
as a way to begin addressing the problem of discrete clustering in general. 
Second, we extend this basic  technique to data consisting of high dimensional but equal length vectors.  
Then, we adapt our technique
to permit unequal length vectors.   Our approach is quite different from others 
and there are so  many approaches that we defer 
discussion of a selection of them to Sec. \ref{review}.

Heuristically, our basic method in low dimensions is as follows.  Start with $n$ vectors of dimension $J$,
say $x_i = (x_{i1}, \ldots , x_{iJ})$ for $i=1, \ldots , n$, 
where,  for $i=1, \ldots , n$ and $j = 1, \ldots , J$, $x_{ij}$  assumes
values in a finite set $A_j = \{ 1, \ldots , a_j\}$ that can be identified with 
the first $a_j = \#(A_j)$ natural numbers.  
For each $j=1, \ldots , J$ we form a $n \times a_j$ membership matrix $M_j$   
by treating each of the $J$ variables separately.   Then the combined membership matrix is 
$M = (M_1, \ldots , M_J)$ of dimension $n \times \sum_{j} a_j$.  Using Hamming
distance on the rows of $M$ gives a dissimilarity measure on the $x_i$'s, say, $d$.  
Now, we can use any hierarchical method to generate a clustering of the $x_i$'s but we 
argue that adding a second stage by ensembling will improve performance.

We add a second stage by first choosing $K_b \sim {\sf DUnif}[2, \sqrt n]$ for $b = 1, \ldots , B$ and assume 
the $K_b$'s are distinct.   For each $K_b$, we use a hierarchical clustering technique 
based on $d$ to produce a clustering with $K_b$ clusters. Each of the $B$ clusterings gives a combined
membership matrix $M_b = (M_b1, \ldots , M_bJ)$ of the form $M$ above. Again, we can 
use the Hamming distance on the rows as a dissimilarity, say $D$. As with $d$, any hierarchical 
clustering method can be used with the ensemble level dissimilarity $D$.

In our work, we found that using $d$ and $D$ with the same linkage criterion
typically gave better results than using $d$ and $D$ with different linkage criteria.
Thus, we have six categorical clustering techniques:  Three use the basic method based on $d$ (i.e., first stage only) with single linkage
(SL), complete linkage (CL), and average linkage (AL), while the other three use the ensemble method based on $d$ and $D$ (i.e.,first and second stages) with SL, CL, or AL.  As a generality, in our comparisons we found that the ensembled clusterings under AL or CL typically gave the best results.

We describe two paradigms in which the sort of ensembling described above is likely to be
beneficial.  One paradigm uses squared error loss and the familiar Jensen's inequality
first used in bagging, an ensembling context for classification, see \cite{Breiman::1996}.  The other paradigm uses a dissimilarity that parallels 
zero-one loss.  Both descriptions rely on the concept of a true clustering
${\cal{C}}_T= (C_{T1}, \ldots, C_{T K_T})$ of size $K_T$ 
based on the population defined by $P$, the probability distribution
generating the data.  
For our results, we assume ${\cal{C}}_T$ is uniquely defined in a formal sense. 


We extend this basic method to high dimensional data vectors of equal length
by partitioning the vectors into multiple subvectors of a uniform but smaller
length, applying our basic method to each of them, and then combining the results 
with another layer of ensembling.

Our second extension, to variable length categorical vectors, is more complicated but can be used to address the clustering of such vectors in genome sciences.  It involves the concept of alignment.   There are various forms of alignment (local, semi-global, 
global, etc.) but the basic idea is to force vectors of categorical data of differing lengths to have the
same length by adding an extra symbol e.g., a $\varnothing$, in strategic places.  
Then the resulting equal length vectors can be clustered as in our first extension.

In the next section, we review the main techniques for clustering categorical data
and indicate where they differ from the methods we have proposed.  
 In Sec. \ref{lowdim}
we formally present our technique described above and provide a series of
examples, both simulated and with real data, that verify our basic technique
works well.   We also present theoretical results that suggest our method
should work well in some generality.  In Sec. \ref{highdim} we present our first extension
and in Sec. \ref{highdimGen} we present our second extension.   In the final Sec. \ref{conclude}, 
we discuss several issues related to the use of our method.

\section{Clustering Techniques for Categorical Data}
\label{review}

In this section, we review six techniques for clustering categorical data,
roughly in the order in which they were first proposed.   


\subsection{$K$-modes}
\label{Kmodes}

$K$-modes, see \cite{Huang:1998}, is an extension of the familiar $K$-means 
procedure for continuous data to categorical data.  However, there 
are two essential differences.  First,
since the mean of a cluster does not make sense for categorical data, the modal 
value of a cluster is used instead; like the mean, the mode is taken
componentwise.  Second, in place of the Euclidean distance, $K$-modes uses
the Hamming distance, again componentwise.    The initial modes are usually 
chosen randomly from the observations.  As recognized in \cite{Huang:1998},
this leads to instability and frequent inaccuracy in that $K$-modes often
gives locally optimal clusterings that are not globally optimal.  

There have been many efforts to overcome the instability and inaccuracy
of $K$-modes clustering with categorical data.   Indeed,
\cite{Huang:1998} suggested choosing the initial modes to be as far from
each other as possible.   Even if this were formalized, it is not clear how it
would ensure the resulting $K$-modes clustering would be accurate or stable.
A different approach was taken in \cite{Wu:etal:2007}.  These authors used
the density of a point $x_i$ defined to be $p(x_i) = 1/\sum_{i\prime=1}^n d(x_i, x_{i\prime})$ so
that a point with high density should have many points relatively close to it.
\cite{Cao:etal:2009} 
used a similar density at points.  
Further attempts to find good initial values are in \cite{Bai:etal:2012}
who use the algorithm in \cite{Cao:etal:2009} with a different distance function.
Also, \cite{Khan:Ahmad:2013} proposed a method for selecting the most relevant 
attributes and clustering on them individually.  Then, they take representatives
of the clusterings as the initial values for a
$K$-modes clustering.

Our method is based on ensembling so it is automatically stable.

\subsection{DBSCAN}
\label{DBSCAN}

There have been several papers using the density-based algorithm, DBSCAN (Density-Based Spatial Clustering of Applications with Noise).  Originally proposed for
continuous data by \cite{Ester:etal:1996}, it extends to categorical
variables because Hamming distance can be used to define a dissimilarity matrix.  
DBSCAN defines a cluster to be a 
maximum set of density-connected points; two points are density connected if and only
if the distance between them is less than a pre-assigned parameter.  This means
that every point in a cluster must have a minimum number of points within a given radius.
There are other approaches that are similar to the DBSCAN such as
 Cactus, see \cite{Ganti:etal:1999}  and Clicks, see \cite{Zaki:etal:2007}.  
These and other methods are used in \cite{Andreopoulos:Wang:2007} on { \sf ZOO} and 
{\sf SOYBEAN} data where it is seen they do not outperform $K$-modes.  Hence,
we do not use Cactus, Clicks or their variants here.  Note that to use these methods
one must choose a distance parameter that influences the size of the resulting clusters.

Our method requires no auxilliary parameter and combines clusterings over randomly selected
dimensions avoiding the question of `density' in a discrete context.

\subsection{ROCK}
\label{ROCK}

\cite{Guha:etal:1999} presented a robust agglomerative hierarchical-clustering algorithm that can 
be applied to categorical data.  It is referred to as ROCK (RObust Clustering using linKs). 
ROCK employs `links' in the sense that it counts the number of point-to-point
hops one must make to form a path from one data point to another.  Note that this
relies on the number of points there are between two selected points, but not directly
on the distance between them.  In essence, ROCK iteratively merges clusters to achieve
high within-cluster similarity.  However, this de facto requires a concept of distance
between points by way of hops -- two hopes being twice as long as one hop.  Our
method does not rely on distance -- it only counts the number of locations at which
two strings differ.
 
\subsection{Hamming Distance (HD)}
\label{HD}

\cite{Zhang:etal:2006} use the Hamming distance (HD) from each observation to a reference position $r$.
Then, they form the histogram generated by the values $d(x_i, r)$ and set up a hypothesis test.
Let ${\cal{H}}_0$ be that there are no clusters in the data and take ${\cal{H}}_1$ to be
the negation of this statement.  Under ${\cal{H}}_0$, we expect the histogram to be
approximately normal, at least for well chosen $r$.
Under the null hypothesis, lack of clustering, \cite{Zhang:etal:2006}
estimates the frequency of the uniform `HD'.   This is compared to the
HD vector with respect to $r$ by way of a Chi-squared statistic.  If the
Chi-square statistic is too large, this is evidence against the null.  So,
the data that make the Chi-square large are removed and the process is repeated.   
This is an iterative method that relies on testing and the choice of $r$.  It is therefore
likely not stable unlike an ensemble method.  

\subsection{Model based clustering (MBC)} 
\label{MBC}

Model based clustering has attracted a lot of attention, see \cite{Fraley:Raftery:2002}.
The main idea is that an overall mixture model for the observations can be identified and the subsets of data can be assigned to components in the mixture.    So, the overall model is of the form
$$
f(x)=\sum_{k=1}^{K} \pi_k p_k(x|\theta_k),
$$
where $\sum_{k=1}^{K} \pi_k=1$, $\pi_k\leq 0$ and the $p_k$'s are the components. The  $\pi_k$  can be used as indicators of what fraction of data are from each component and the $p_k$'s can be used to assign data to components. The likelihood of a mixture model is 
$$
\mathfrak{L}(\theta_1,\ldots,\theta_K;\pi_1,\ldots,\pi_K)=\prod_{i=1}^n \sum_{k=1}^K \pi_k p_k(x|\theta_k),
$$
where the parameters can be estimated using the Expectation-Maximization algorithm. 
When the data are categorical, the $p_k$'s should be categorical and 
\cite{Celeux:Govaert:1991} proposed
$$
p_k(x\mid \theta_k)=\prod_{j=1}^A\prod_{h=1}^{m_j} (\theta_k^{jh})^{x^{jk}},
$$
where $\theta_k=(\theta_k^{jh},h=1,\ldots,m_j,j=1,\ldots,J)$, $\sum_{h=1}^{m_j} \theta_k^{jh}=1$  and $\theta_k^{jh}$ gives the probability of category $h$ of the variable $j$ in the cluster $k$. The proposed model is actually the product of $A$ conditionally independent categorical distributions.  Several versions of this method
are implemented in the R package { \sf Rmixmod}. 

Model based clustering should work well -- and does work well for continuous data.  For
categorical data it is not at all clear when a mixture model holds or even is a good approximation.  
In fact, the mixture model likely does not hold very often and we would only expect good 
performance from MBC when it does.  By contrast, ensembling
should be perform reliably well over a larger range of DG's.

\subsection{Ensemble approaches}
\label{ensembles}

Another approach is to regard the categorical data as the result of a clustering procedure.  
So, one can form a
matrix in which the rows represent the $n$ data points and each of the columns represent
the values of an attribute of the data point.  If the attributes are taken as
cluster labels then the clusterings, one for each attribute,
can be used to form a consensus clustering by defining a dissimilarity
between data points using Hamming distance.  This process is known
as ensembling and was first used on categorical data by \cite{He:etal:2005}
via the techniques  CSPA, HGPA and MCLA.   From these, \cite{He:etal:2005}
select the clustering with the greatest Average Normalized Mutual Information (ANMI) 
as the final result.  This seeks to merge information among different clusterings.
However the mutual information is  measure of dependence and finding clusterings
that are dependent is not the same as finding clusterings that are accurate.  
They apply their technique to several data sets but note
that it does not perform well on unbalanced data like {\sf ZOO}.
The actual process of ensembling is reviewed in \cite{Strehl:Ghosh:2002}.  We find that
ensembling over the dissimilarities is a better way to ensemble since it seems to give a more
accurate assessment of the discrepancy between points.

The idea of evidence accumulation clustering (EAC) is due to \cite{Fred:Jain:2005} and 
is an ensembling approach that 
initially was used for continuous data.  The central idea is to create many clusterings of different 
sizes (by $K$-means) that can be pooled via a `co-association matrix' that 
weights points in each clustering according to their membership in each cluster.   This 
matrix can then be easily modified to give a
dissimilarity so that single-linkage clustering can be applied, yielding a final clustering.  
The first use of EAC on categorical data seems to be \cite{Iam-On:etal:2012}.
They used the $K$-modes technique with random initializations for cluster centers
to generate $B$ base clusterings.  Then they ensembled these by various methods.
In Section \ref{Num}, we use this approach for categorical data but using $K$-modes 
instead $K$-means (EN-KM). The same idea can be used to implement EAC on the 
MBC (EN-MBC).   Again, our method ensembles over dissimilairties rather than 
clusters directly.  It seems that in catgorical clustering getting a good way to assess
discrepancy between data points provides better solutions than trying to merge
different clusterings directly.

\section{Basic Technique in Low Dimensions}
\label{lowdim}

Here we present our ensemble technique for clustering categorical data, provide some justification for
why it should perform well, and see how this is borne out in a few examples.

\subsection{Formal presentation}
\label{basictechnique}
 
To fix notation, we assume $n$ independent and identical (IID) outcomes $x_i,~ i=1,\ldots,n$ 
of a random variable $X$.  The $x_i$'s are assumed $J$-dimensional and written as
$(x_{i1}, \dots , x_{iJ})$ 
where each $x_{ij}$ is an element of a finite set $A_j$; the $A_j$'s do not have to be the same
but $\sup_j \#( A_j) <  \infty$ is required for our method.

We denote a clustering of size $K$ by 
${\cal{C}}_K = (C_{K1}, \ldots , C_{KK})$; often we assume that for each $K$ only one
clustering will be produced.   Without loss of generality, we assume each $A_j$ is identified with
$a_j$ natural numbers.


We write the $x_i$'s as the rows in an observation matrix and look at its columns:
\begin{eqnarray}
O = \left( \begin{array}{c}
x_1  \\
\vdots \\
x_n \end{array} \right) 
= \left( \begin{array} {ccc} 
c_1 & \ldots & c_J  \\ \end{array}  \right) .
\end{eqnarray}
That is, $c_j = (x_{1j}, \ldots , x_{nj})^T \in \varprod_{i=1}^n A_j$.  

Now, for each $c_j$ form the $a_j \times n$ membership matrix $M_j$
with elements
\begin{eqnarray}
m_{\ell, i} = \begin{cases}
1 &    x_{i j} = \ell\\
0  & \hbox{else} 
\end{cases}
\end{eqnarray}
where $\ell = 1, \ldots , a_j$ and $i = 1, \ldots , n $.  
The combined membership matrix is
$$
M = \left( \begin{array}{ccc}
M_1^T & \ldots & M_J^T \\ \end{array} \right),
$$
the result of concatenating $J$ matrices with dimensions $n \times a_1, \ldots , n \times a_J$.
Thus, $M$ is $n \times A$ where $A=\sum_{j=1}^J \vert a_j \vert$.  
We use the transpose of the $M_j$'s
so that each row will correspond to a subject.
Rewriting $M$ in terms of its rows i.e., in terms of subjects 1 through $n$, gives
\begin{eqnarray}
M = \left( \begin{array}{c}
v_1 \\
\vdots \\
v_n \end{array}  \right).
\end{eqnarray}
Now, consider the dissimilarity $D = ( d(v_\nu , v_\mu) )_{\nu, \mu = 1, \ldots , n}$ where
$$
d(v_\nu , v_\mu) = \sum_{\tau = 1}^J \delta(v_{\nu \tau}, v_{\mu \tau}) .
$$
Here, $\delta(a,b) = 0$ if $a=b$ and one otherwise.
Under this dissimilarity one can do hierarchical clustering under single linkage (HCSL),
average linkage (HCAL), and complete linkage (HCCL).

This gives three clustering techniques for categorical data.
However, we do not advocate them as we have described because, as will be seen below, we can get better 
performance by ensembling and treating sets of outlying points representing less than $\alpha100 \%$ of
the data separately (see \cite{Amiri:etal:2015} for the basic technique
to handle potential chaining problems and outliers under single 
linkage context).

Our ensembling procedure is as follows.  Choose one of HCSL, HCAL, or HCCL   
 and draw $K_b \sim {\sf DUnif}[ 2 , \sqrt{n}]$
for $b=1, \ldots , B$.  For each value of $b$ use the chosen method to
find a clustering of size $K_b$.  Then form the incidence matrix
\begin{eqnarray}
\mathcal{I}=\left[  \begin{array}{c|c|c|c}
w_{12} & w_{12} & \ldots & w_{1B} \\
\vdots & \vdots & \vdots & \vdots \\
w_{n1} & w_{n2} & \ldots & w_{nB} 
\end{array} \right]
\label{Incmatrix}
\end{eqnarray}
Note that each column corresponds to a clustering and $w_{ib}$ is the index of the cluster in the $b$-th
clustering to which $x_i$ belongs.    For any of the three linkages, we can find the
dissimilarity matrix
\begin{eqnarray}
T = ( d_B(x_i, x_j) )_{i,j = 1, \ldots , n}  ,
\label{endis1}
\end{eqnarray}
in which
\begin{eqnarray}
d_B(x_i, x_j) = \frac{1}{B}\sum_{b=1}^B \delta( w_{ib}, w_{jb}).
\label{endis2}
\end{eqnarray}
Note that we only compare the entries in a row within their respective columns.  Thus, we 
never compare, say $w_{11}$ with $w_{12}$ but only compare $w_{11}$ to
entries of the form $w_{i1}$ for $i \neq 1$.
Now, we can use the same linkages on $D$ as before (SL, AL, CL) with the
treatment of small sets of outliers as in \cite{Amiri:etal:2015}
to get a final clustering for any given $K$.  (The choice of $K$ can also be estimated by any of a number
of techniques; see \cite{Amiri:etal:2015} for a discussion and comparison of such techniques.)

\subsection{Why and when ensembling works} 
\label{whyensemble}

  Consider two IID observations $\mathbb{X}$ and $\mathbb{Y}$ from the population
  defined by $P$ assumed to have a well defined true clustering ${\cal{C}}_T$.  Write the cluster membership function as
  $$
  C_T(x) = k \Longleftrightarrow x \in C_{T k}.
  $$
  Now define the {\it similarity} of  $C_T(\mathbb{X})$ and $C_T(\mathbb{Y})$ to be the indicator function
\begin{eqnarray}
\chi\big(C_T(\mathbb{X}),C_T(\mathbb{Y})\big) = \begin{cases}
1 &    C_T(\mathbb{X})=C_T(\mathbb{Y}) \nonumber\\
0  & \hbox{else} 
\end{cases}
\end{eqnarray}
This gives the (random) dissimilarity $\delta = 1 - \chi\big(C_T(\mathbb{X}),C_T(\mathbb{Y})\big)$.  Clearly 
$\exists P_\delta \in (0,1)$ so that
$$
E(\delta)=P\big(C_T(\mathbb{X}) \neq C_T(\mathbb{Y})\big)=P_\delta.
$$
Empirically, we use data ${\cal{D}}$ to form a clustering denoted $\phi = \phi(\cdot \mid {\cal{D}})$ where
$\phi(x \mid {\cal{D}})$ is the index of the cluster in ${\cal{C}}({\cal{D}})$ containing $x$.  So, 
  $$
  \widehat \delta= 1 -\chi\big(\phi(\mathbb{X},{\cal{D}}),\phi(\mathbb{Y},{\cal{D}})\big),
  $$ 
  in which $\mathbb{X}$ and $\mathbb{Y}$ are independent of the IID data ${\cal{D}}$.
 Now, given $B$ independent replications of the clustering technique we obtain
  $$\widehat \delta_{k_b}=1- \chi\big(\phi_{k_b}(\mathbb{X},{\cal{D}}),\phi_{k_b}(\mathbb{Y},{\cal{D}})\big), \quad
  b=1,\ldots,B.
  $$
So, we can form the dissimilarity
 \begin{eqnarray}
 \bar{\delta}_{EN}=1-\frac{1}{B}\sum_{b=1}^{B}\chi\big(\phi_{k_b}(\mathbb{X},{\cal{D}}),\phi_{k_b}(\mathbb{Y},{\cal{D}})\big),
 \label{endis3}
\end{eqnarray}
 where $k_b \overset{iid}{\sim} DUnif(K_\ell,K_u)$, $K_\ell$ and $K_u$ are  the minimum and maximum acceptable sizes for a candidate clustering.   Note that the expression in \eqref{endis3} reduces to \eqref{endis2}
when $\mathbb{X}=x_i$ and $\mathbb{Y}=x_j$.  Thus, \eqref{endis2} is a specific instance of \eqref{endis3} so
when we do the hierarchical clustering in Sec. \ref{basictechnique} we are using a dissimilarity of
the same form as \eqref{endis3}.  Consequently, any statement involving $\bar{\delta}_{EN}$
will lead to a corresponding statement for the ensembling in our clustering technique.

The following proposition shows that the ensemble dissimilarity is closer to the true value of the dissimilarity
on average than any individual member of the ensemble 
when the clusterings in the ensemble are `mean unbiased'.
Formally, a clustering method is mean unbiased if and only if the expectation of the empirical
dissimilarity, conditional on the data and number of clusters, is $P_\delta$ and
so  independent of ${\cal{D}}$ and $k_b$.  That is,  a clustering method is mean unbiased if and only if
\begin{eqnarray}
E(\widehat \delta_{k_b}|{\cal{D}},k_b)=P\big(C_T(\mathbb{X})\neq C_T(\mathbb{Y})\big).
\label{meanunbiased2}
\end{eqnarray}
The intuition behind this definition is that the probability that two data points are
not in the same cluster is independent of $k_b$ for any $k_b$ chosen between $k_\ell$ and $k_u$
for fixed data ${\cal{D}}$.
Note that, in this expression, the expectation is taken over the collection
of possible clusterings holding ${\cal{D}}$ and $k_b$ fixed.  This requires that
some auxiliary random selections, such as the random
selection of initial cluster centers in $K$-means, is
built into the clustering technique.  When a clustering technique is
mean unbiased we get that
\begin{eqnarray}
P_\delta = E(\widehat \delta_b|{\cal{D}},k_b) = E(\bar{\delta}_{EN} |{\cal{D}},k_b), 
\label{endis4}
\end{eqnarray}
and hence the ensembling preserves the mean unbiasedness.
In this case, we can show that the ensembling behaves like taking a 
mean of IID variables in the
sense that the variance of the average of the $\widehat \delta_b$'s 
decreases as $1/B$.  The interpretation of this is that the ensemble dissimilarity
is a more accurate representation of the dissimilarity between any
two data points than the dissimilarity from any individual
clustering in the ensemble.
 
 \begin{proposition}
 Assume all the clusterings $\phi_{k_b}$ are mean unbiased.  Then, for
 each $D$ and set $k_1, \ldots , k_B$ there
 is a $B^*$ so that $B \geq B^*$ implies
 \begin{eqnarray}
 P( ~ \mid \bar{\delta}_{EN}-E\delta ~ \mid ~ >\epsilon ~ \mid ~ {\cal{D}},k_1,\ldots,k_B) \leq  
 \min_b P( ~ \mid \widehat \delta_{k_b}-E\delta ~ \mid  ~>\epsilon ~ \mid ~ {\cal{D}},k_b).
 \label{ensemblebd}
 \end{eqnarray}
 \end{proposition}
 \begin{proof}
 The expected dissimilarity for any clustering function $\phi$ is
 \begin{eqnarray}
 E (\widehat \delta \mid {\cal{D}},k)=P\big(\phi(\mathbb{X},{\cal{D}})=\phi(\mathbb{Y},{\cal{D}}) |{\cal{D}},k\big)=P_{\delta}({\cal{D}},k),
 \label{empdis}
\end{eqnarray}
 i.e., possibly dependent on ${\cal{D}}$ and the number $k$ of clusters in $\phi$.
 Consequently, the conditional probability
\[
P\big(|\widehat \delta-E(\widehat \delta|{\cal{D}})|>\epsilon|{\cal{D}},k\big) > 0,
\]
is well defined for each ${\cal{D}}$ and $k$.  Since $k_b$ is
chosen from a finite set of integers, 
\[
\min_b P\big(|\widehat \delta_{k_b}-E(\widehat \delta_{k_b}|{\cal{D}})|>\epsilon|{\cal{D}},k_b\big) >0,
\]
i.e., the right hand side of \eqref{ensemblebd} is strictly positive.

Next observe that in general, for each $k_b$ there is a $P_\delta({\cal{D}}, k_b)$
as in \eqref{empdis}.  
Since the clusterings in the ensemble are independent over reselected $k_b$'s,  
and the variance of a Bernoulli($p$) is uniformly bounded 
by 1/4 for $p \in [0,1]$, we have that
\begin{eqnarray}
P(|\bar{\delta}_{EN}-E\bar{ \delta}_{EN}|>\epsilon|{\cal{D}},k_1,\ldots,k_B) &\leq& \frac{1}{\epsilon^2}      E( \mid  \bar{ \delta}_{EN}-E\bar{\delta}_{EN} \mid ^2 |{\cal{D}},k_1,\ldots,k_B) \nonumber \\ 
&=&
\frac{1}{\epsilon^2 B^2}     E \bigg( \big( \sum_{b=1}^{B} \widehat \delta_{k_b}-\sum_{b=1}^{B}   E(\widehat \delta_{k_b}|{\cal{D}},) \big)^2 \bigg| {\cal{D}},k_1,\ldots,k_B \bigg) \nonumber\\
&=&\frac{1}{\epsilon^2 B^2}     \sum_{b=1}^{B}  E \bigg( \big(  \widehat \delta_{k_b}-  E( \widehat\delta_{k_b} |{\cal{D}},k_b)  \big)^2 \bigg| {\cal{D}},k_b \bigg) \nonumber\\
&=&
\frac{1}{\epsilon^2 B^2}     \sum_{b=1}^{B}  P_\delta({\cal{D}},k_b)(1-P_\delta({\cal{D}},k_b))\leq \frac{1}{4\epsilon^2 B}.
\label{bernoullibd}
\end{eqnarray}

Hence, for $B$ sufficiently large, \eqref{bernoullibd} implies
\begin{eqnarray}
P(|\bar{\delta}_{EN}-E\bar{\delta}_{EN}|>\epsilon|{\cal{D}},k_1,\ldots,k_B) \leq P(|\widehat \delta_{b}-E\widehat \delta_{b}|>\epsilon|{\cal{D}}, k_b).
\nonumber
\end{eqnarray}
The difference between this bound and the statement of the result is that
this bound uses $E \bar{\delta}_{EN}$ and $E \widehat{\delta}_b$ in place of $E \delta$.
Since $\widehat{\delta}_b$ is mean unbiased, \eqref{endis4} implies that for $B$ large enough
\eqref{bernoullibd} gives
\begin{eqnarray}
P(|\bar{\delta}_{EN}-E\delta|>\epsilon|{\cal{D}},k_1,\ldots,k_B) \leq \frac{1}{4\epsilon^2 B} 
\leq \min_b P\big(|\widehat \delta_{k_b}-E \delta \mid  >\epsilon|{\cal{D}},k_b\big) .
\label{endis5}
\end{eqnarray}
\end{proof}
\noindent
Note that even though mean unbiasedness implies that $E \delta$ is independent of $D$ and $k_b$, $\bar{\delta}_{EN}$ and $\widehat{\delta}_{k_b}$ need not be.  So, the conditioning in 
\eqref{endis5} cannot be dropped.  
If we take expectations over $D$ for finite $n$, the proof of the proposition can be modified to give
$$
E_{\cal{D}} P(|\bar{\delta}_{EN}-E\delta|>\epsilon|{\cal{D}},k_1,\ldots,k_B)
\leq
\min_b E_{\cal{D}} P\big(|\widehat \delta_{k_b}-E \delta \mid  >\epsilon|{\cal{D}},k_b\big) .
$$
It can also be seen that if $n$ increases and the clustering $\phi(\cdot \mid {\cal{D}})$
is consistent in the sense that $\phi(x \mid {\cal{D}}) \rightarrow C_T(x)$ pointwise in $x$, then for an nondecreasing
sequence of $B^* = B^*(n)$, the inequality continues to hold and the conditioning on ${\cal{D}}$ drops out.

Instead of arguing that the mean dissimilarity from ensembling behaves well, one can argue that 
the actual clustering from an ensemble method such as we have presented is close to the population
clustering.  That is, it is possible to identify a condition that ensures ensemble clusterings are more
accurate on average than the individual clusterings they combine.  To see this, fix any
sequence $k_1, \ldots , k_B$ and let $K_\ell \leq k \leq K_u$ and write
\begin{eqnarray}
\Delta &=&\chi(C_T(\mathbb{X}),\phi(\mathbb{X},{\cal{D}})), \nonumber \\
E \Delta({\cal{D}},k) &=&P\big(C_T(\mathbb{X})=\phi(\mathbb{X},{\cal{D}})|{\cal{D}}, k \big), \nonumber\\
\Delta_{EN} &=& \frac{1}{B} \sum_{b=1}^B \chi(C_T(X), \phi_{k_b}(X , {\cal{D}}) ), \nonumber \\
E(  \Delta_{EN} \mid {\cal{D}}, k_1, \ldots , k_B) &=&  \frac{1}{B} \sum_{b=1}^B      P\big(C_T(\mathbb{X})=\phi_{k_b}(\mathbb{X},{\cal{D}})|{\cal{D}},k_b\big).   \nonumber
\end{eqnarray}
Since the $k_b$'s are drawn IID from DUnif$[K_\ell , K_u]$, the condition
\begin{eqnarray}
\frac{1}{K_{u}-K_{\ell}} \sum_{i=K_{\ell}}^{K_{u}} E \Delta({\cal{D}}, i)   > E\Delta({\cal{D}},k)
\label{avgind}
\end{eqnarray}
is equivalent to
$$
E_{k_b} E(\Delta_{EN} \mid {\cal{D}}, k_1, \ldots , k_B) > E(\Delta \mid {\cal{D}}, k).
$$
Moreover,
$$
V\big(E(\Delta_{EN}|{\cal{D}},k_1,\ldots,k_B)\big)\leq V\big(E(\Delta|{\cal{D}},K)\big) .
$$
That is, \eqref{avgind} is equivalent to saying the ensemble clustering is more accurate on average 
and has a smaller variance than any of its constituents.  Condition \eqref{avgind} defines a subset
of $ {\cal {D}} \times [K_\ell, K_u]^B$ on which ensembling
good clusterings will improve the overall clustering.  This is a parallel to the set defined in \cite{Breiman::1996} and quantifies the fact that while ensembling generally gives better results
than not ensembling, it is possible that ensembling in some cases can give worse results.
That is, the best clustering amongst the $B$ clusterings may be better than the ensemble clustering, but
only infrequently.  This is borne out in our numerical analyses below.

Although Hamming distance is the natural metric to use with discrete, nominal data, it is easier to see
that ensembling clusterings gives improved performance using squared error loss.
Our result for clustering is modeled on Sec 4.1  \cite{Breiman::1996} for bagging classifiers.

Let $E_{\cal{D}}$ be the expectation operator for the data generator that produced $D$.  Then,
the population-averaged clustering is 
\begin{eqnarray}
\phi_A(\mathbf{x})=E_{\cal{D}}(\phi(\mathbf{x},{\cal{D}})).
\label{nok}
\end{eqnarray}
Implicitly, this assumes $k$ is fixed and not random.
This is an analog to the ensemble based clustering presented in Sec. \ref{basictechnique}.
Two important squared error `distances' between the true clustering and estimates of it are
\begin{eqnarray}
APE(\phi)&=&E(C_T- \phi(\mathbb{X},{\cal{D}}))^2, \nonumber\\
APE( \phi_A)&=&E(C_T- \phi_A(\mathbb{X}))^2. \nonumber
\end{eqnarray}
A Jensen's inequality gives that the squared error loss from using $\phi_A$ is smaller than
from using $\phi$.

\begin{proposition}
Let $\mathbb{X}$ be an independent outcome from the same distribution as generated ${\cal{D}}$.   Then,
$$
E_{\cal{D}}(APE( \phi)) \geq APE( \phi_A).
$$
\end{proposition}
\begin{proof}

\begin{eqnarray}
APE( \phi_A)&=&E_{\mathbb{X}}(C_T(\mathbb{X})- \phi_A(\mathbb{X}))^2\nonumber\\
		&=&E_{\mathbb{X}}[E_{\cal{D}}(C_T(\mathbb{X})-\phi(\mathbb{X},D))^2] \nonumber \\
		&\leq& E_{\mathbb{X}}E_{\cal{D}}[(C_T(\mathbb{X})- \phi(\mathbb{X},{\cal{D}}))^2]\nonumber\\
		&=& E_{\cal{D}}APE( \phi).\nonumber
\end{eqnarray}
\end{proof}
Obviously, the inequality might be equality, or nearly so, in which case the averaging provides little to no gain.
On the other hand, the averaging may make the distribution of $\phi_A$ concentrate around $C_T$
more tightly than the distribution of $\phi$ does, in which case the inequality would be strict
and the difference between the two sides could be large representing a substantial gain from the
ensembling.  In this sense, the ensembling may stabilize $\phi$ as a way to reduce its variance and hence
improve its performance. Note that if the expectation in \eqref{nok} were taken over $K=k$, $\phi$, or
both $K$ and $\phi$ as well,
the Jensen's inequality argument would continue to hold with $E_{\cal{D}}$ replaced by $E_K E_{\cal{D}}$,  $E_\phi E_{\cal{D}}$,
or $E_K E_{\cal{D}} E_\phi$. That is, more averaging can only improve the clustering.

\subsection{Numerical analysis} 
\label{Num}

To evaluate our proposed methodologies, we did two numerical analyses, one with simulated data
and the other with real data.   Technically, the real and simulated datasets are classification data but we applied our clustering techniques ignoring the class labels.  Thus, the question
is how well the clustering techniques replicated the known true classes.  To assess this we
defined the classification rate (CR) to be the proportion of observations from a data set that were correctly assigned
to their cluster or class.  To calculate the CR, we start by generating the clustering from the
data.  Then, we order the clusters according to how much they overlap with the correct clusters.
For instance, if the estimated clusters are $\hat{C}_1, \hat{C}_2, \hbox{and} \hat{C_3}$ and the correct
clusters are $C_1^*, C_2^*, \hbox{and}, C_3^*$ then, if necessary, we relabel the estimated clusters so that
$\hat{C}_1 = \arg \max \{ \hat{C}_j \cap C^*_1 \mid j = 1, 2, 3\}$, 
$\hat{C}_2 = \arg \max \{ \hat{C}_j \cap C^*_2 \mid j = 1, 2, 3\}$, and
$\hat{C}_3 = \arg \max \{ \hat{C}_j \cap C^*_3 \mid j = 1, 2, 3\}$.   In this way, we ensure that each of the
estimated clusters overlap maximally, in a sequential sence, with exactly one of the true clusters.    This 
can be done using the Hungarian algorithm, see \cite{Huhn::2005} for details.

The simulated data were generated as follows.   First, we fixed the dimension $J=20$.  Then we generated sets of
20-dimensional data points of sizes 50 or 125 assuming different cluster structures.  For instance, in the equi-sized cluster case with five clusters with $n=125$ we generated five clusters of size 25.   For each such cluster,
we proceeded as follows.
Choose  $p_1, \ldots , p_{5} \sim {\sf Unif}(0.2, 0.8)$ IID and choose $a_1 \sim DUnif(3,20)$. Then, draw 25 values from ${\sf Bin}(a_1,  p_1)$
where $a_1$ is the number of values the variable giving the first dimension can assume.  Do this again
for each of ${\sf Bin}(a_1, p_j)$ for $j=2, \ldots, 5$.  Taken together these 125 values give the
first entries for the 125 vectors of length 20 to be formed.  We proceed similarly to obtain the second entries for
the 125 vectors of dimension 20, but draw a new $a_2 \sim DUnif(3, 20)$ independent of $a_1$.
That is, we choose values 
$p_j \sim  {\sf Unif}(0.2, 0.8)$ IID for $j=1, \ldots , 5$ and again generate 25 values
from each ${\sf Bin}(a_2, p_j)$ for each $j$.  Doing this 18 more times for IID ranges of length 
$a_3 = \#(A_3), \ldots, a_{20} = \#(A_{20})$ gives 125 vectors of length 20 that represent 
five clusters each of size 25.  Doing this entire
procedure 3000 times gives 3000 such data sets.  We did this sort of procedure for various clustering patterns
as indicated in Table \ref{clusterpatterns}, e.g., taking nine values in place of 25 values in $D_2$ for the
first cluster.  The patterns were chosen to test the methods systematically over a 
reasonable range of true clusterings based on the size of the clusters.

\begin{table}[h] \center 
\caption{Design of simulated data.} 

\begin{tabular}{lccc} \hline\label{clusterpatterns}
Name&$K$&$(n_1,n_2, \ldots)$&$n$\\ \hline
$D_1$&5&(25,25,25,25, 25)& 125\\
$D_2$&5&(9,29,29,29,29)& 125\\
$D_3$&5&(10,10,35,35,35)&125\\
$D_4$&5&(10,10,10,47,48)&125\\
$D_5$&5&(10,10,10,10,85)& 125\\
$D_6$&5&(10,25,25,25,40)&125\\
$D_7$&5&(10,10,30,30,45)&125\\
$D_8$&5&(10,10,10,35,60)&125\\
$D_9$&5&(10,10,25,40,40)&125\\
$D_{10}$&2&(25,25)& 50\\
$D_{11}$&2&(15,35)& 50\\ \hline
\end{tabular}\\
 { \footnotesize  The top row labeled $D_1$ shows the equi-sized clusters for $n=125$.
Data sets indicated by $D_2$ through $D_5$ are to be understood as contexts for testing how the methods perform
when there are five clusters with two sizes, one small and one large.  Data sets indicated by $D_6$ through $D_9$ 
are to be
understood as contexts for testing how the methods perform when there are five clusters of three sizes, 
small, medium, and large.
Data sets indicated by $D_{10}$ and $D_{11}$ are to be understood as contexts for testing how the methods
perform when there are two clusters, one equisized and the other consisting of a small and a large cluster.}

\end{table}

For each simulation scenario, we applied 12 methods.  The first four ($K$-modes, DBSCAN, ROCK, MBC)
are as described in Subsecs. \ref{Kmodes}, \ref{DBSCAN}, \ref{ROCK}, and \ref{MBC}.  The next two are ensembled versions of
$K$-modes and MBC, as described in Subsec. \ref{ensembles}.   The last six methods are agglomerative.  The first three are
linkage based using Hamming distance but not ensembled.  The last three are the same except they have
been ensembled.  These six methods are also described in \ref{basictechnique}.  We did not include HD from Subsec. \ref{HD} in this comparison because it required the estimate of a reference position $r$ and the
implementation we had did not provide one.

 The results are given in the Table \ref{tabsimres}.     The entries in bold are the maxima in their column.
In some cases, we bolded two entries because they were so close as to be indistinguishable statistically.
The entries with asterisks in each column are the next best  methods, again we asterisked entries that
seemed very close statistically.  The pattern that emerges with great clarity is that HCAL and ENAL perform best.   This is even borne out in the mean column -- although we note that the mean is
merely a summary of the CRs.  It does not have any particular meaning because the cluster structures
were chosen deterministically.  The next best method after HCAL and ENAL is MBC.  Note that $D_8$ provided minimal 
discrimination over the techniques that were not bolded, i.e., not the best, and the mean 
column has the same property.  HCCL and ENCL also performed better than the worst techniques for this 
collection of examples, but performed noticeably worse than the best.

\begin{table}[h]
\footnotesize
\centering
   \caption{Classification rates of twelve methods using simulated data.}
\begin{tabular}{lcccccccccccc} \hline\label{tabsimres}
&\multicolumn{11}{c}{Data}&
\\ \cline{2-12} 
method&$D_1$&$D_2$&$D_3$&$D_4$&$D_5$&$D_6$&$D_7$&$D_8$&$D_9$&$D_{10}$&$D_{11}$&mean\\ \hline
$K$-modes&0.56&0.45&0.46&0.46&0.42&0.46&0.47&0.45&0.47&0.81&0.78&0.53\\
DBSCAN&0.44&0.38&0.43&0.51&0.68&0.41&0.46&0.54&0.44&0.65&0.77&0.52\\
Rock&0.71$^*$&0.471&0.45&0.48&0.51&0.45&0.46&0.51&0.46&\bf 0.95&0.76&0.56\\
MBC&0.71$^*$&0.55$^*$&0.51$^*$&0.45&0.37&0.52$^*$&0.51$^*$&0.44&0.50&0.91&0.86&0.57\\
EN-KM &0.45&0.28&0.33&0.44&0.67&0.35&0.39&0.51&0.37&0.74&0.82&0.49\\
EN-MBC&0.32&0.28&0.35&0.49&0.67&0.35&0.40&0.53&0.39&0.87&0.87&0.50\\
HCSL&0.25&0.26&0.31&0.41&0.71$^*$&0.34&0.38&0.50&0.35&0.57&0.73&0.43\\
HCAL&\bf0.85&\bf0.67&\bf0.71&\bf0.75&\bf0.81&\bf0.67&\bf0.72&\bf0.76&\bf0.71&\bf0.96&\bf0.96&\bf0.78\\
HCCL&0.51&0.45&0.49&0.55\bf&0.59&0.45&0.50&0.56$^*$&0.50$^*$&0.72&0.71&0.55\\
ENSL&0.24&0.27&0.31&0.41&0.71$^*$&0.34&0.38&0.50&0.35&0.57&0.74&0.44\\
ENAL&\bf 0.88& \bf 0.68& \bf 0.70& \bf 0.69& \bf 0.79&\bf0.68&\bf0.71&\bf 
0.75&\bf0.72&\bf0.96&\bf 0.96&\bf0.77\\
ENCL&0.55&0.46&0.50$^*$&0.54$^*$&0.54&0.46&0.50$^*$&0.55$^*$&0.50$^*$&0.72&0.71&0.55\\
\hline
\end{tabular}
\end{table}

Next, we turn to the comparisons of the methods on real data.   We use five different data sets
as summarized in Table \ref{tabreal}.  All  can be downloaded from the UC Irvine Machine Learning
Repository (\cite{Lichman::2013}).  
For these data sets we were
able to implement HD, as it gave a reasonable $r$ for 
three of the data sets ({\sf ZOO}, {\sf SOYBEAN - small},
and {\sf CANCER}), so we compared 14 methods rather than 13.  For the remaining two data sets ({\sf Mushroom} and {\sf LYMPHO}) we manually found an $r$ that gave credible results.\\

\begin{table}[htp] \center
   \caption{Key features of five real data sets.} 
\begin{tabular}{lcccc}  \hline\label{tabreal}
Name&$K$&$J$&$n$&$\max_{j=1}^J a_j$\\ \hline
\sf ZOO&7&16& 101&7\\
\sf SOYBEAN-small&4&35&47&7 \\ 
\sf Mushroom$^*$ &2&21&400 &7\\
 \sf Lymphography domain (LYMPHO) &4&18& 148&8 \\
\sf  Primary tumor domain (CANCER)&21&17& 339&21 \\
\hline
\end{tabular}\\
{ \raggedright \footnotesize  \noindent This table lists the five data sets used in our analysis giving the true number of clusters in each,
the dimension of the data sets, the sample sizes, and the maximum number of possible
distinct discrete values for each, i.e., the $\max_{j=1}^J a_j$'s for each. 
* \footnotesize Since the dataset is large, we only used the last 400 observations in our analysis.}
\end{table}

The results of our analyses are presented in Table \ref{realdataCR}.  As before we bolded the best methods 
with respect to CR and asterisked the second best methods.  Clearly, the best methods were ENCL and ENAL and their performances were similar.  The next best methods were MBC and, possibly, HD.   

The results in Table \ref{realdataCR} also show differences from those in
Table \ref{tabsimres}.   First, the ensembling of the SL, AL, and CL methods tends to improve them
substantially for the real data whereas it has only a slight effect on the
simulated data.  HCAL does well on simulated data, but poorly on the real data. ENCL does well on the real data
but poorly on the simulated data.  Ensembing of $K$-modes and MBC actually 
makes them worse for the simulated data and only slightly better for the real data.  We note that
ENSL, ENAL and ENCL are in terms of Hamming distance and according to Proposition 1, ensembling 
should routinely give more accurate clusterings so it is no surprise that the ensembled versions 
perform better.    Moreover, according to Proposition 2 and the discussion in Sec. \ref{whyensemble}, 
regardless of accuracy, the ensembled version of a clustering method has less variability, as can be seen for 
$K$-modes and MBC when compared with their ensembled versions.  This is seen in Table \ref{realdataCR}; numbers in parentheses are standard deviations (SD) and the absence of a number in parentheses 
next to an entry indicates an SD of zero. 

\begin{table}[h]
\centering
   \caption{Classification rates of the proposed and comparison methods using real data.}
\begin{tabular}{lcccccc}  \hline 
&\multicolumn{5}{c}{Data}&
\\ \cline{2-6}  \label{realdataCR}
method&\sf  ZOO& \sf SOYBEAN&\sf  MUSHROOM &\sf  LYMPHO &\sf  CANCER&mean\\ \hline
$K$-modes&0.72(0.08)&0.80(0.12)&0.68(0.12)&0.45(0.05)&0.37(0.02)&0.60\\
DBSCAN&0.88&\bf 1&0.50&0.54&0.32&0.65\\
Rock&0.87&\bf 1&0.51&0.40&0.15&0.59\\
HD&\bf 0.95&\bf 1&0.66&0.59$^*$&0.28&0.70\\
MBC&0.88(0.06)&\bf 1& \bf 0.97&0.47(0.03)&0.32(0.03)&0.72\\
EN-KM&0.76&0.78$^*$&\bf 0.98&0.55&0.28&0.67\\
EN-MBC&0.74&0.79$^*$&0.97&\bf 0.63&0.33&0.69\\
HCSL&0.88&\bf 1&0.50&0.57&\bf 0.38&0.66\\
HCAL&0.89&\bf 1&0.51&0.58$^*$&0.35$^*$&0.66\\
HCCL&0.91$^*$& \bf 1&0.51&0.58$^*$&0.29&0.66\\
ENSL&0.88&\bf 1&0.73&0.57&0.28&0.69\\
ENAL&0.89&\bf 1&\bf 0.97&0.58$^*$&\bf 0.38&\bf 0.76\\
ENCL&0.91$^*$&\bf 1&\bf 0.97&\bf 0.64&0.35$^*$&\bf 0.77\\
\hline
\end{tabular}\\ 
\end{table}

The inference from our numerical work is that, in some general sense, the best option is to choose ENAL.
Recall, ENAL gave the best performance on the simulated data and nearly the best performance on
the real data -- only ENCL is better and only by 0.01 in terms of the means.  ENAL performs noticeably 
better on the simulated data than on the real data, possibly because
the real data is much more complicated.  We explain the better performance of ENAL 
by recalling \eqref{meanunbiased2} and observing that it is likely to be small for AL for many data sets
because AL tends to give compact clusters that are stable under ensembling.  This follows 
because AL is relatively insensitive to outliers and an average is more stable than an individual outcome.
Thus, if two points $x_1$ and $x_2$ are in the same cluster in a clustering of size $K_1$, 
it is unlikely that reasonable AL reclusterings of size $K_2$ will put $x_1$ and $x_2$ in different clusters even 
when $K_2$ is far from $K_1$.  Hence, AL is likely to put $x_1$ and $x_2$ in the same cluster even for
clusterings of different sizes and this will be preserved under ensembling, cf. the remarks after \eqref{meanunbiased2}.

\section{Extension to High Dimensional Vectors}
\label{highdim}

Extending any clustering technique to high dimensions must evade the Curse of Dimensionality to be
effective.  The way the Curse affects clustering, loosely speaking, is to make the distance between 
any two vectors nearly the same.  This has been observed in work due to \cite{Hall:etal:2005}
but the basic idea can be found in \cite{Murtagh::2004}.  The earliest statement of the Curse of Dimensionality
in a clustering context seems to be \cite{Beyer:etal:1999}.  However, all of these are in the context of
clustering continuous, not categorical, variables.  Clustering categorical variables is fundamentally
different because in categorical clustering there is usually no meaningful way to choose a distance whose
numerical values correspond to physical distances.
For instance, in DNA one could assign values 1, 2, 3, and 4 to $A$, $T$, $C$, and $G$. However,
it is not reasonable to say $A$ and $C$ are twice as far apart as $A$ and $B$.  For this reason,
Hamming distance, which merely indicates whether two outcomes are the same, is preferred.

\subsection{Ensembling over subspace clusterings}

Our first result states a version of the Curse of Dimensionality for the clustering of high dimensional
categorical data.  

\begin{proposition}
\label{dist1}
Let $X= (X_1, \ldots , X_J)^T$ and $X^\prime= (X^\prime_1, \ldots , X^\prime_J)^T$ be independent 
$J$ dimensional categorical random variables.  Let 
$$
d(X, X^\prime) =\frac{1}{J} \sum_{j=1}^J \delta(X_j , X^\prime_j)
$$
Then, as $J \rightarrow \infty$,
$$
E(d(X, X^\prime)) \rightarrow C \quad \hbox{and} \quad \hbox{Var}( d(X, X^\prime)) \rightarrow 0,
$$
for some $C \geq 0$.  In particular, if $X_j$ and $X_j^\prime$ have the same distribution for
all $j$, there is a $p \geq 0$ so that $E(\delta(X_j , X^\prime_j)) = p$ and
$$
E(d(X, X^\prime)) =p \quad \hbox{and} \quad \hbox{Var}( d(X, X^\prime)) = \frac{p(1-p)}{J}.
$$
\end{proposition}
The proof is little more than the law of large numbers and heuristically means that in the limit of high 
dimensions any two independent vectors
of categorical variables are equidistant.   The rate of decrease in variance controls how 
quickly this occurs.  Hence, clustering based on distances
necessarily degenerates or, more precisely, the output of a clustering procedure will be random
in the sense that no one clustering can reasonably be favored over another.
Prop. \ref{dist1} is a discrete analog of  \cite{Kriegel:etal:2009}, 
\cite{Steinbach:etal:2004},  and \cite{Parsons:etal:2004}.

The two main ways analysts evade the Curse in high dimensional clustering with continuous data 
are subspace clustering, i.e., find a clustering on a subset of the variables and `lift' it to a clustering of
the entire vector, and feature selection, in which clustering is done on relatively few functions of the variables
thought to characterize the clusters in the data.

Evidence has accumulated that feature selection does not work very well on continuous data -- absent
extensive knowledge about which features are relevant -- see \cite{Yeung:Ruzzo:2001}, \cite{Chang::1983},
\cite{Witten:Tibshirani:2010}.  Indeed, it can be verified that if generic methods for obtaining features,
e.g., PCA, are used with categorical data, the computing demands become infeasible.
Since techniques based on feature selection are even harder to devise
and compute for discrete data, feature selection does not seem a promising approach to high 
dimensional clustering of categorical data.
However, various forms of subspace clustering have shown some promise (\cite{Friedman:Meulman:2004}, \cite{Jing:etal:2007}, \cite{Witten:Tibshirani:2010}, and  \cite{Bai:etal:2011}).  The common feature of
these subspace methods is that they seek to identify a single subspace from which to generate
a good clustering.

We refine the idea of subspace clustering by adding a layer of ensembling over randomly chosen
subspaces.  That is,
we randomly generate subspaces, derive clusterings for each of them, and then combine them
analogous to the procedure that generated \eqref{Incmatrix}.  Our procedure is as follows.

Recall that the vector of categorical variables is
 $ x \in A_1 \times \ldots \times A_J$ and write $J=hR$ for some integers $h$ and $R$.  Then, each
 $x$ can be partitioned into $R$ subvectors of length $h$.   So, each data point of the form
 $x$ gives $R$ submatrices of dimension  $n \times h$ and the whole data set $x_1, \ldots, x_n$ can be represented by
 matrices of the form
$$
L_r (n)= \begin{pmatrix} x_{1, (r-1)h +1} &  \ldots & x_{1, rh} \\
\vdots & \vdots & \vdots \\
 x_{n, (r-1)h +1} & \ldots & x_{n, rh}
 \end{pmatrix} .
 $$
Now, each $L_r$ can be clustered by any of the techniques in Sec. 3.1.

Recall, the conclusion of Sec. 3.1 was that ENAL gave the best results overall.
Let the result of, say, ENAL on $L_r(n)$ be denoted ${\cal{C}}(K_r) = \{C_{r1}, \ldots , C_{r K_r} \}$,  
for $r=1, \ldots, R$.
This gives $R$ clusterings of the full data set denoted by 
$$
\{{\cal{C}}(K_1), \ldots ,{\cal{C}}(K_R) \}.
$$
These clusterings can be combined by using the dissimilarity in \eqref{endis2}
and then any hierarchical method can be used to obtain final clusters.
In our computing below, we chose AL because it did slightly better than CL in Sec. \ref{Num}.
In our work below, we refer to this method as subspace ensembling.

This basic template can obviously be extended to the case where the equi-sized sets $L_r$ 
of variables
are chosen randomly.  This is an improvement because it means that variables do
not have to be adjacent to contribute to the same clustering.  One obvious way to proceed is to
choose random subsets of size $h$ without replacement until $R$ subsets are obtained. 
In addition, $h$ itself can vary across subsets in which case the
requirement that $J = hR$ is replaced by the requirement that $J=h_1 + \dots + h_\nu$ for
some randomly chosen $\nu$. We denote this form of subspace ensembling by $WOR$. 
 Note that even if
two variables are useful in the same clustering they will rarely be chosen to contribute to the same clustering.
A second way to proceed is where $h$ varies across subsets but the subsets of variables are chosen with replacement and then 
duplicate variables are removed from each $L_r(n)$; we denote
this method of subspace ensembling by $WR$. This permits variables that contribute usefully to the same clustering to have a fair chance of being used in the same clustering.  On the other hand, $WR$ does not require each variable to contribute to
the overall clustering.  For computational efficiency, the $WR$ method we used had an extra layer of subsampling
and discarding of repeated variables. So, if we denote by $L_r(n)'$ the result of removing duplicates from $L_r(n)$ as described, we again chose with replacement and removed duplicates from each $L_r(n)'$. This ensured that the clusters in the ensemble were based on few enough variables that results could be found with reasonable running times.
One way to see that this double resampling approach is likely to reduce the dimension of each subspace is demonstrated by the following.

\begin{proposition}
Suppose a bootstrap sample of size $J$ is taken from a set of $J$ distinct objects denoted 
$\{1, \ldots , J\}$.  

1) Let $N$ be the number of distinct elements from $\{1, \ldots , J\}$ in
such a sample.  Then for $1 \leq k \leq J$, 
$$
P( N = k) =
\frac{
\binom{J}{k}\big( \sum_{a_1, \ldots , a_k \mid \sum_{m=1}^k a_m =J, a_m >  0} 
\binom{J}{a_1~~\ldots~~a_k}\big)}{J^J}.
$$

2) Now suppose a second level of bootstrap sampling is done on the $N=k$ distinct elements.  Let
$N^*$ be the number of distinct elements in the second level bootstrap sample.
Then, for $1 \leq k^* \leq k$,
$$
P(N^* = k^* \mid N = k) =
\frac{
\binom{k}{k^*}\big( \sum_{a_1, \ldots , a_{k^*}  \mid \sum_{m=1}^{k^*} a_m =k, a_m >  0} 
\binom{k}{a_1~~\ldots~~a_{k^*}}\big)}{k^k}.
$$

\end{proposition}

\begin{proof} For 1) the cardinality of the event $\{ N = k\}$ is
\begin{eqnarray}
&& \#( \mbox{ {ways to get $k$ distinct elements from $J$ draws from $J$ distinct objects}}) \nonumber \\
&=&
\#( \mbox{ { subsets of size $k$ from a set of $J$ distinct objects}}) \nonumber \\
&& \times ~
\#( \mbox{ {ways to allocate the $J$ draws to the $k$ values }})
\nonumber \\
&=& \binom{J}{k} \times
\sum_{a_1, \ldots , a_k \mid \sum_{m=1}^k a_m =J, a_m >  0} 
\binom{J}{a_1~~\ldots~~a_k} .
\label{onelevel}
\end{eqnarray}
Now, the expression for $P(N=k)$ in 1) follows by dividing \eqref{onelevel} by $J^J$, the number of points
in the sample space. For 2), merely note that, conditional on $N=k$, the probability that $N^* = k^*$ is the same
as \eqref{onelevel} but with $J$ replaced by $k$ and $k$ replaced by $k^*$.
\end{proof}

Clearly, $N^* \leq N$ and so $E(N^*) \leq E(N)$. Also, expressions for $E(N)$ and $E(N^* \mid N=k)$ follow from Prop. 4 but are not edifying.  However, given these,
$$
E(N^*) = \sum_{k=1}^J \sum_{k^*=1}^k k^* P(N^*=k^* \mid  N=k) P(N=k)
$$ 
is relatively easy to find. The value of $E(N)/J$ for $J > 1000$ is about 0.63 in agreement with the
fact that the probability of a given $x_i$ not being chosen is $(1- 1/n)^n \rightarrow 1/e$
as $n \rightarrow \infty$ and $.63 \approx 1 - (1/e)$.   For the double bootstrapping, $E(N^*)/J \approx 0.47$
can be verified computationally.


\subsection{Numerical evaluation of subspace ensembling }
 
 We compare four methods using simulated data and two real data sets.  The four methods are the 
 proposed techniques $WR$ and $WOR$ with $M=200$ ensembles, the mixed weighted $K$-modes 
 ($MWKM$) method due to \cite{Bai:etal:2011}, and the sparse hierarchical method ($SHM$) due to 
 \cite{Witten:Tibshirani:2010}.  \cite{Bai:etal:2011} only tested their method on dimensions up to 68; 
 we use their method as implemented in code supplied by the authors for 4027, 44K, and 50k dimensions.
$SHM$ was intended for continuous data; we implemented a version of code supplied by the authors 
that we modified for categorical data to use Hamming distance. Unfortunately, in our computations 
$SHM$ performed poorly so the results are not presented here. \\

\subsubsection{Simulations}
\label{sims}

Our simulation study involves data that mimics genetic sequence data.
Consider a sample of size  $n=\sum_{i=1}^K n_i$ where $K=5$ and $n_i$ is the number of observations 
from the $i$-th cluster, regarded as if it were a population in its own right as well as a component in the 
larger population from which the sample of size $n$ was drawn.    We assume there are $J=50,000$ categorical
variables each drawn from the set ${\cal{S}}=\{A,T,C,G\}$.   Then, to simulate data, it remains to assign distributions 
to the 50,000 variables.

First, we consider data sets denoted $\mathcal{D}_1$ and formed as follows.
Suppose  the first $q_{1}$ dimensions for $n_1$ samples are drawn IID from the set $\Sigma=\{A,T,C,G\}$  
where $P_C=P_G=\frac{1}{3}$, $P_A=P_T=\frac{1}{6} $ and the first $q_1$ dimensions for the 
other subsamples $n_i, i=2,\ldots,5$ are drawn IID from $\Sigma$ uniformly.  Similarly, for the next
$q_2$ dimensions, for the $n_2$ samples draw values using the same distribution as for the first $q_1$ variables
for $n_1$ and use the uniform on $\Sigma$ for $n_1$, $n_3, n_4, n_5$.
Repeat this procedure to fill in the $q_3$, $q_4$ and $q_5$ dimensions for the $n$ data points 
so the data have $K=5$. 
To add noise to data, we take the last $q_6$ variables to be IID uniform from $\Sigma$ as well.
For the purposes of evaluating the various methods, we choose the number of variables randomly by
setting $q_1,\ldots, q_6\sim {\sf Mnom}(50000, p=(0.15, 0.15, 0.15, 0.15, 0.15, 0.25))$.
To generate $\mathcal{D}_2$ we increased the number of noise variables by taking
$q_1,\ldots, q_6\sim {\sf Mnom}(50000, p=(0.1, 0.1, 0.1, 0.1, 0.1, 0.5))$.

Table \ref{simhighdim} shows the mean CRs for $WOR$, $WR$, and $MWKM$ for 500 data sets
of the form $\mathcal{D}_1$ and $\mathcal{D}_2$ for each setting of the numbers of samples in the
five clusters.   Obviously, $WR$ substantially outperforms the
other methods.  We attribute this to the fact that $WR$ permits variables that are related
to each other to recur in the ensembling so that the basic method from Sec. \ref{lowdim}
can find the underlying structure in the subspaces.  Note also that the performance of $WR$ is
insensitive to the cluster size.

\begin{table}[htp]
\centering
   \caption{Classification rates of three methods using simulated data.}
\begin{tabular}{lccccc} \hline
&\multicolumn{2}{c}{$n=(10,10,10,10,10)$}&&\multicolumn{2}{c}{$n=(5,10,10,10,15)$}
\\ \cline{2-3} \cline{5-6} 
method&$\mathcal{D}_1$&$\mathcal{D}_2$&&$\mathcal{D}_1$&$\mathcal{D}_2$\\ \hline
WOR&0.501&0.473&&0.564&0.488\\
WR&0.998&0.989&&0.997&0.987\\
MWKM&0.615&0.583&&0.615&0.587\\
\hline
&\multicolumn{2}{c}{$n=(5,5,13,13,14)$}&&\multicolumn{2}{c}{$n=(5,5,10,15,15)$}
\\ \cline{2-3} \cline{5-6} 
WOR&0.504&0.589&&0.531&0.589\\
WR&0.974&0.998&&0.978&0.996\\
MWKM&0.606&0.574&&0.6128&0.578\\
\hline
&\multicolumn{2}{c}{$n=(5,5,5,17,18)$}&&\multicolumn{2}{c}{$n=(5,5,5,10,25)$}
\\ \cline{2-3} \cline{5-6} 
WOR&0.566&0.678&&0.622&0.692\\
WR&0.977&0.996&&0.968&0.995\\
MWKM&0.597&0.577&&0.555&0.541\\
\hline
&\multicolumn{2}{c}{$n=(5,5,10,10,20)$}&&\multicolumn{2}{c}{$n=(5,5,5,5,30)$}
\\ \cline{2-3} \cline{5-6} 
WOR&0.552&0.631&&0.688&0.752\\
WR&0.976&0.998&&0.962&0.995\\
MWKM&0.617&0.570&&0.514&0.487\\
\hline
\end{tabular}\\ 
\label{simhighdim}
\end{table}

\subsubsection{Asian rice (Oryza sativa)}
\label{sativa}

We consider a population consisting of five varieties of rice
(Oryza sativa) and use clustering on single nucleotide polymorphism (SNP) data to assess the plausibility of the division
of the species into five varieties.   The data, {\sf RICE}, were originally presented and analyzed in 
\cite{Zhao:etal:2011} and consist of 391 samples from the five varieties indica (87), aus (57), temperate 
japonica (96), aromatic (14) and tropical japonica (97) where the numbers in parentheses indicate
the number of samples from each variety.    


The analysis done in \cite{Zhao:etal:2011} was to measure the genetic similarity between individuals.
Essentially,  \cite{Zhao:etal:2011} calculate the proportion of times a pair of nucleotides at the same position
differ.   Mathematically, this is equivalent to using a version of the average Hamming distance. Note that
in their analysis they ignored missing values as is permitted in PLINK, \cite{Purcell:etal:2007}.
Setting $K=5$ we regenerated their analysis and dendrogram.  The result is shown on the left 
hand panel of Fig. \ref{fig1:level}.  Around the outer ring of the
circle the correct memberships of the data points are indicated. The CR for this clustering is one.\\

For comparison, the right hand panel in Fig. \ref{fig1:level} shows the dendrogram
for clustering the {\sf RICE} data using $WR$, again setting $K=5$.   
The CR was found to be one.   No dendrogram for MWKM can be shown because
it is not a hierarchical method.  However, the CR for MWKM was 0.87, making it second best
in performance.  Even though PLINK and $WR$ have the same CR, visually it is obvious that
$WR$ gives the better dendrogram because the clusters are more clearly separated.
That is, $WR$ does not perform better in terms of correctness but does provide a better visualization
of the data.   This is the effect of ensembling over dissimilarity matrices.

    \begin{figure} \centering
      \includegraphics[width=.6\textwidth]{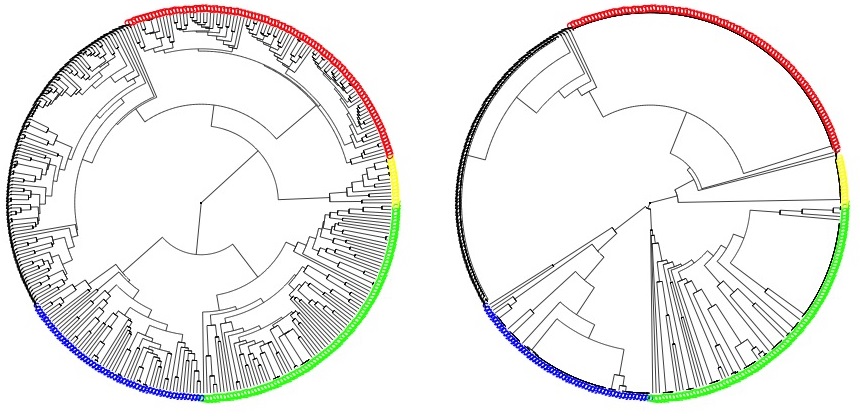} \caption{Circular dendrogram generated by
      the same technique as in \cite{Zhao:etal:2011}(left)
      and $WR$ (right), for the {\sf RICE} data. 
      Green, blue, red, yellow, and black indicate indica, aus,  temperate japonica,  aromatic  
         and tropical japonica, respectively. }
      \label{fig1:level}
      \end{figure}
      
   

\subsubsection{Gene expression data}

In this example we demonstrate that the performance of $WR$ can be regarded as robust. Consider the gene expression data presented and analyzed in \cite{Alizadeh:etal:2000}.  It is actually
classification data and analyzed as such in \cite{Dudoit:etal:2002}.  Here, for demonstration purposes we 
compare some of the clusterings we can generate to the known classes so as to find CRs.  
The sample size is 62 and there are three classes:  diffuse large 
B-cell lymphomas (D) with 42 samples, follicular lymphoma (F) with 9 samples, and chronic lymphocytic leukemia (C) with 11 samples.
The dimension of the gene expression data after pre-processing is 4026.  (The pre-processing included
normalization, imputation, log transformation, and standardization to zero mean and unit variance across genes.)

First we applied $K$-means to the data 1000 times with $K=3$ and random starts, and found an average
CR of 0.79  with a standard error of 0.13.   When we applied $MWKM$ to the data, we found its CR to be 0.67,
noticeably worse than $K$-means.

Now consider the following procedure.  Trichotomize all 4026 variables by using their 33\% and 66\% percentiles
and relabel them as `1', `2', and `3'.  Then, apply $WOR$ and $WR$ to the discretized data.
To get CRs for hierarchical methods, we must convert their dendrograms to clusterings by
cutting them at some level.  To do this, we used the function {\sf cuttree} (with $K=3$) in
{\sf R}.
The resulting CRs 
for $WOR$, $WR$ were 0.71 and 0.84, respectively.  That is, even when the data are discrete only
because they were made that way artificially, $WR$ handily outperformed three other methods.

\begin{figure}[h]\centering 
\includegraphics[scale=.35]{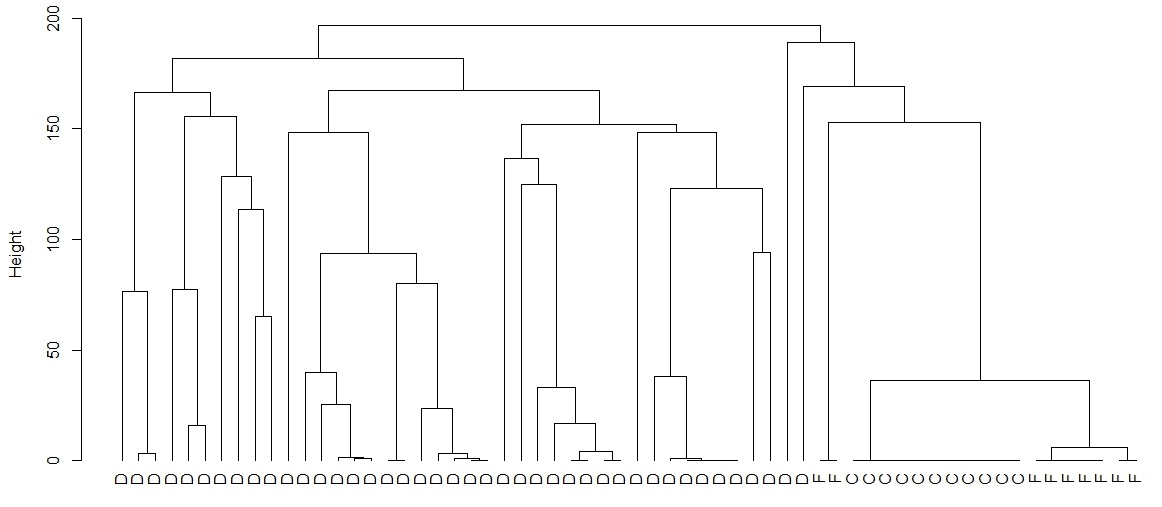}\caption{Dendrogram of WR.}
\label{DendGeneExp}
\end{figure}

The full dendrogram from $WR$ is given in Fig.  \ref{DendGeneExp}.
Along the horizontal axis, the correct labels of the classes are given.  If these were ignored and one were
merely to eyeball the data, one could be led to put the two rightmost D's and the first two F's into the
cluster of C's, giving an CR of 58/62 = .93 -- much higher than .84.
One could just as well put the leftmost D's into
one cluster, the next 13 D's into a second cluster, the next eight D's into a third cluster, 
and the rightmost 18 observations into a fourth cluster, leaving the intervening data points essentially as a fifth cluster
that does not cohere.  In this case, the CR would be terrible.  So, even though the data are
artificially discretized, using an automated method of
$WR$ on a discretized $WR$ and {\sf cutree} gives a result in the midrange of what
informal methods would give.  This is evidence that ensemble methods such as $WR$
are inherently robust.   Otherwise put, reading dendrograms informally can be misleading whereas
formal methods may be reliably accurate.

\section{Extension to high dimensional vectors of unequal length}
\label{highdimGen}

We extend our method to clustering categorical vectors
of different lengths.  This is an important  clustering problem in genomics
because it is desirable to be able to cluster strains of organisms, 
for instance, even though their genomes have different lengths 
in terms of number of nucleotides.  The first step 
is to preprocess the data so all the vectors 
have the same length.  This process is called alignment.  
The aligned vectors
can then be clustered using the technique of Sec. \ref{highdim}.
The point of this section is to verify that our clustering method
is effective even after alignment.

To be specific, consider sequence data of the form
${\cal{X}}_i = (x_{i1}x_{i2}...x_{iJ_i})$ with $i=1, \ldots , n$ in which
each $x_{ij}$ for $j=1, \ldots , J_i$ is a nucleotide in 
$\{A, T, C, G\}$.  It is obvious that a sufficient
condition for our method to apply is that all the $x_{ij}$'s 
assume values in sets $A_j$ for which $\#(A_j)$ is bounded.  
To find a common value for
the $n$ sequences we align them using software called MAFFT-7 (\cite{Katoh:Standley:2013}).  MAFFT-7 is a
multiple alignment program for amino acid or nucleotide sequences.
The basic
procedure was first presented in \cite{Needleman:Wunsch:1970}.  A recent
comparison of  algorithms and software for this kind of alignment problem
was carried out in \cite{Katoh:Standley:2013} who argued that MAFFT-7
is faster and scales up better than other implementations such as
CLUSTAL and MUSCLE.

At the risk of excessive oversimplification, the 
basic idea behind alignment procedures is as follows.  
Suppose two sequences $x_i$ and $x_k$
of different lengths are to be aligned.  Then, the alignment procedure
introduces place holders represented by $\phi$ so that the two sequences
are of the same length and the subsequences that do match are in the same
place along the overall sequence.  When more than two sequences must
be aligned, a progressive alignment can be used, i.e., two sequence are aligned and fixed, the third one aligned to previous ones, and the procedure continues until all sequences are aligned.
Given that a collection of genomic sequences have been aligned, we can
cluster them by applying our technique. 

In the absence of established theory for this more complicated case, we present two examples to verify that the procedure gives reasonable results.
Both of our examples concern viruses:  Their genomes are large enough
to constitute a nontrivial test of our clustering method and of different enough in lengths from species to species that alignment of
some sort is necessary. 

As a first simple example, consider the virus family {\it Filovirdae} .  This family includes numerous related viruses that form filamentous infectious viral 
particles (virions).  Their genomes are represented as single-stranded negative-sense RNAs.  The two members of the family that are best known are {\it Ebolavirus} and {\it Marburgvirus}. 
Both viruses, and some of their lesser known relatives, cause severe disease in humans and 
nonhuman primates 
in the form of viral hemorrhagic fevers (see \cite{Pickett:etal:2012}). \\
There are three genera 
in {\it Filoviridae}, and we chose as our data set all the complete and distinct viral
genomes with a known host from this family available from ViPR.  There were 103 in total
from 3 genera, namely, {\it Cuevavirus} (1, Cue), {\it Ebolavirus} (80), and {\it Marburgvirus} (22, Mar) where
the indicators in parentheses show the frequency and the abbreviation.
{\it Ebolavirus} further subdivides into five species: {\it Bundibugyo virus} (3, Bun), 
{\it Reston ebolavirus} (5, Res), {\it Sudan ebolavirus} (6, Sud), {\it Tai Forest ebolavirus} (2, Tai), 
and {\it Zaire ebolavirus} (64, Zai).
The hosts are human, monkey, swine, guinea pig, mouse, and bat (denoted
hum, mon, swi, gpi, mou, bat, respectively, on the dendrograms). The minimum and maximum genome lengths are 18623 and 19114.  While we recognize that the genomes
in the pathogen virus resource are not drawn independently 
from a population, we 
can nevertheless apply our method and evaluate the results.

We apply only $WR$ as $WOR$ and $MWKM$ performed poorly on
high dimensional data. In addition, we used $HCSL$, 
$HCCL$, and $HCAL$ because we had a dissimilarity that could be used after alignment.
Essentially, as long as the distance between $\phi$ and the nucleotides 
could be
omitted from the Hamming distance sum, the dissimilarity was well-defined.
It turned out that all three gave nearly identical results although $HCAL$ was
slightly better. The important point is that $WR$ performed better
than the three non-ensemble methods because the ensembling over dissimilarity matrices gives a better assessment
of the distance between aligned genomes than not ensembling. 

The results for $HCAL$ are shown in the  Fig. \ref{Flohcal}.  
It is seen that the 
 subpopulations are well separated.
 Virologically speaking, this means that the various species correspond to relatively
 tight clusters.  Within the {\it Marburg} cluster (on the left) it appears that most of the genomes
 have either humans or bats as hosts suggesting that one organism (probably the bats)
 is transmitting {\it Marburg} to the other (probably human).  Sudan ebolavirus mainly afflicts humans
 (pending more data)  while {\it Reston ebolavirus} is known not to be a pathogen for humans.
 The vast majority of the { \it Zaire} and {\it Bunibugyo} genomes have human as host. 
 
 Fig. \ref{Flowr} shows the corresponding dendrogram using $WR$,
 an ensemble method.
Qualitatively the results are the same as for $HCAL$.  The improvement of $WR$ over $HCAL$ is
seen in the fact that $WR$ reveals greater separation between the clusters.  Indeed, even if one
corrects for the vertical scale, the leaves within a cluster under $WR$ separate from each other
at a much finer level.  That is, the ensembling over the dissimilarities accentuates the 
differences between genomes in different clusters as discussed in Sec. \ref{whyensemble}.

\begin{figure}[h!]
 \centering
  \includegraphics[scale=.3]{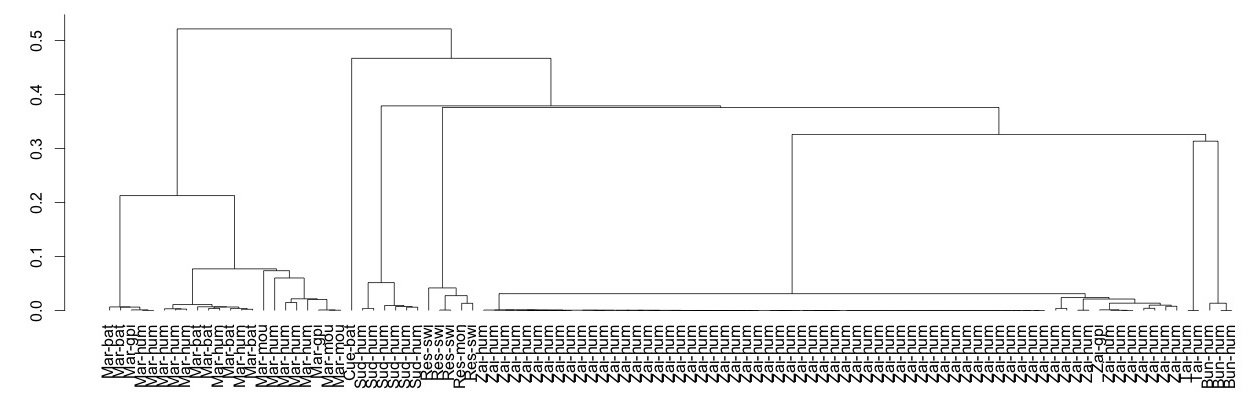}
   \caption{Dendrogram for the {\it Filoviridae} data set using $HCAL$.  The leaves are labeled with the Genus or species and host.  }
   \label{Flohcal}
  \includegraphics[scale=.3]{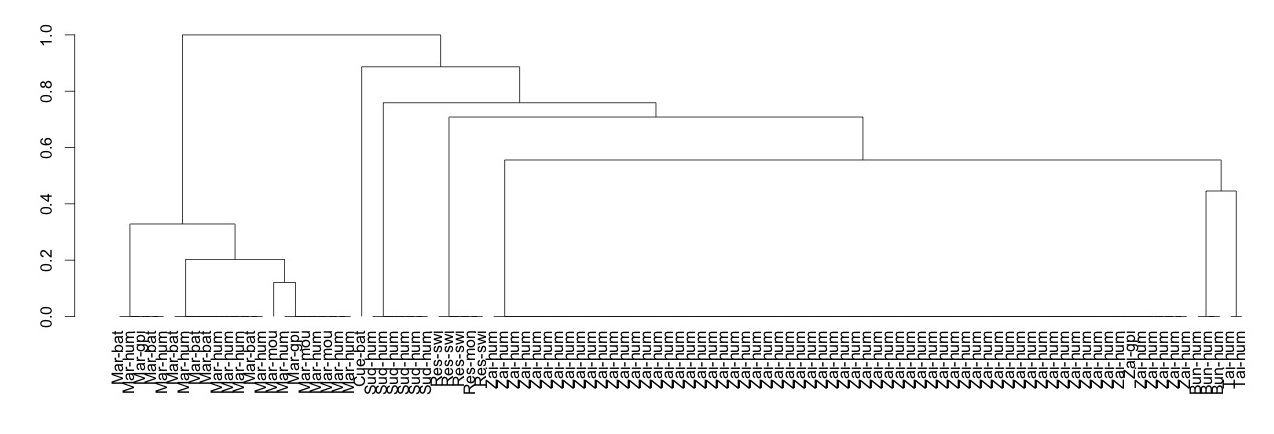}
  \caption{Dendrogram for the {\it Filoviridae} data set using $WR$ under the same labeling
   convention as the Fig. \ref{Flohcal}. }
   \label{Flowr}
   \end{figure}
   
   These dendrograms can be contrasted with a phylogenetic tree for the {\it Filoviridae} viruses.
   Figure \ref{phylotree_Filo} shows the phylogenetic tree generated by the neighbor-joining (NJ)
   method as implemented in the R package {\sf Ape} (\cite{Paradis:etal:2004}).  The NJ method constructs a tree by
   successive pairing of the neighbors. The idea behind a phylogenetic tree, as opposed to a 
   dendrogram, is to represent the sequence of evolutionary steps through which organisms mutated as a reasonable way to classify the existing and extinct organisms.
   The goals of the two sorts of trees are somewhat different and one would not expect them to agree fully, since clustering only gives a
   mathematically optimal path to the evolutionary endpoint while phylogenetic trees try to track genomic changes.
   For instance, the phylogenetic tree shows that {\it Zaire ebolavirus} with a human host separates early into 
   two distinct groups which may or may not be reasonable evolutionarily and is different
   from Fig. \ref{Flowr}.  {\it Tai Forest ebolavirus} and
   {\it Bunidbugyo virus} genomes are seen to be possibly close evolutionary but are 
   not close in Fig. \ref{Flowr}.  On the other hand, {\it Sudan ebolavirus} and {\it Reston ebolavirus} are seen to be
   close in terms of both clustering and phylogenetics while {\it Marburg} is 
   a separate and recent genus, consistent with it being its own cluster.

 \begin{figure}[h]
   \centering
   \includegraphics[scale=.4]{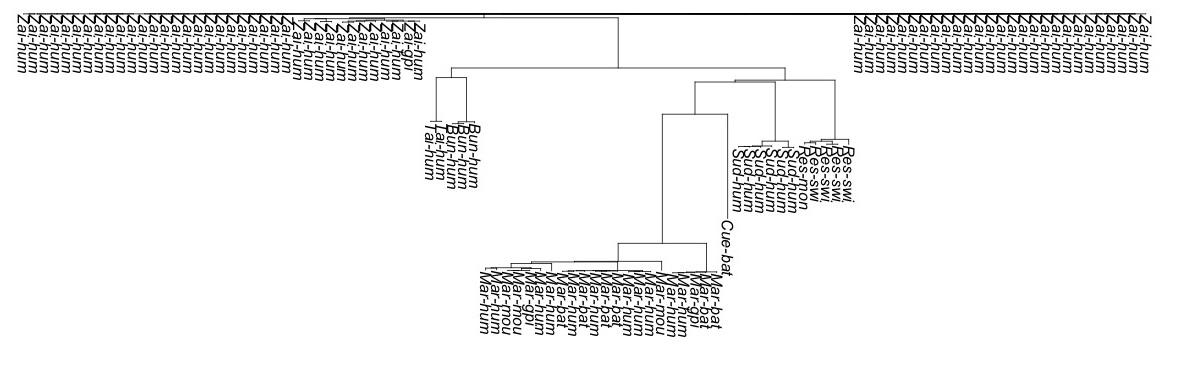}
      \caption{Phylogenetic tree generated for { \it Filoveridae} by NJ. It is broadly similar to the dendrograms
      in Figs. \ref{Flohcal} and \ref{Flowr}, but differs in some details. }
 \label{phylotree_Filo}
   \end{figure}

As a second and more complicated example, we studied the {\it Herpesviridae} family of viruses
that cause diseases in humans and animals.  {\it Herpesviridae} is a much larger family than 
{\it Filoviridae} and the genomes in {\it Herpesviridae} are generally longer as well as more varied in
length than those in {\it Filoviridae}.  
According to ViPR, the family {\it Herpesviridae} is divided into three
subfamilies ({\it Alphaherpesvirinae}, {\it Betaherpesvirinae} and {\it Gammaherpesvirinae}).  
We limited our analysis to the distinct and complete genomes in
{\it Alphaherpesvirinae} that have known hosts; {\it Alphaherpesvirinae} has more complete genomes than
either {\it Betaherpesviridae} or {\it Gammaherpesvidae}.   Within {\it Alphaherpesviridae} there are
has five genera: {\it Iltovirus} (IIt), {\it Mardivirus} (Mar), {\it Scutavirus}, {\it Simplexvirus} (Sim), and {\it Varicellovirus} (Var).
Since {\it Scutavirus} did not have complete any complete genomes, we disregarded this
genus.  The rest remaining genera had 20, 18, 20, 40 genomes, respectively, from different hosts, namely, human, Monkey, chicken, Turkey, Duck, cow, Bat (Fruit), equidae (horse),  Boar, Cat family, Amazona oratrix (denoted
hum, mon, chi, tur, duc, cow, bat, equ, boa, cat, and ora,  respectively in the dendrograms). These viral genomes have lengths ranging from 124784 to 178311 base pairs.

Parallel to the {\it Ebolavirus} example, we present the two dendrograms corresponding to
$HCAL$ and $WR$.  These are in Fig.\ref{Herphcal} and \ref{Herpen}.   The top panel shows that $HCAL$, the
non-ensembled version based on Hamming distance, is qualitatively the same as the
lower panel.  As before, the key difference is that $WR$ yields a cleaner separation of clusters
relative to $HCAL$.   It is important to note that the clusters in
the dendrograms correspond to (genus, host)  pairs.   That is, the clustering corresponds to
identifiable physical differences so the clusters have a clear interpretation.

 \begin{figure}[h!]
 \centering
  \includegraphics[scale=.3]{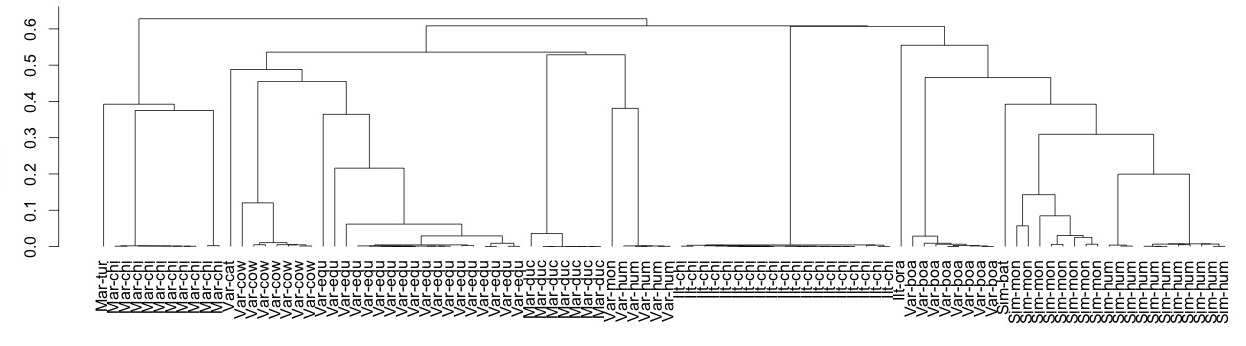}
  \caption{Dendrogram for the {\it Alphaherpesviridae} data set using $HCAL$.  The leaves are labeled with the genome Genus or subpopulation and host.}
   \label{Herphcal}
 \includegraphics[scale=.3]{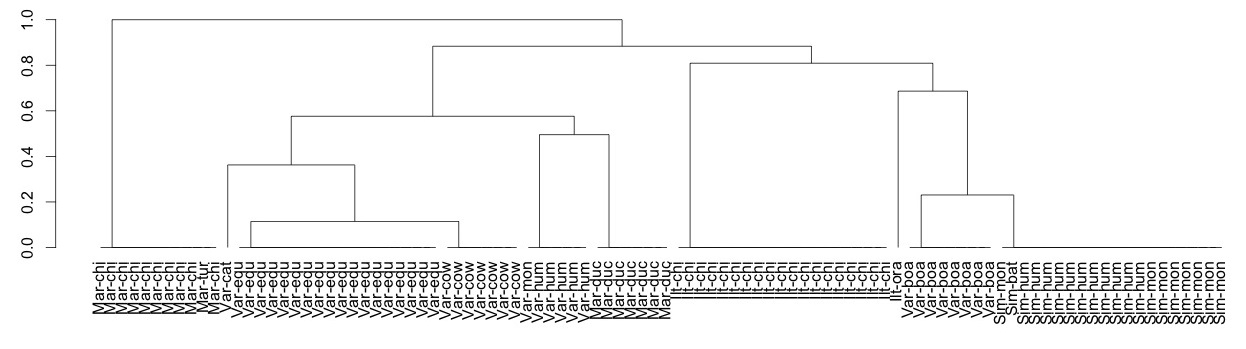}
 \caption{Dendrogram for the {\it Alphaherpesviridae} data set using $WR$ under the same labeling
    convention as the top panel.}
 \label{Herpen}
   \centering
  \includegraphics[scale=.3]{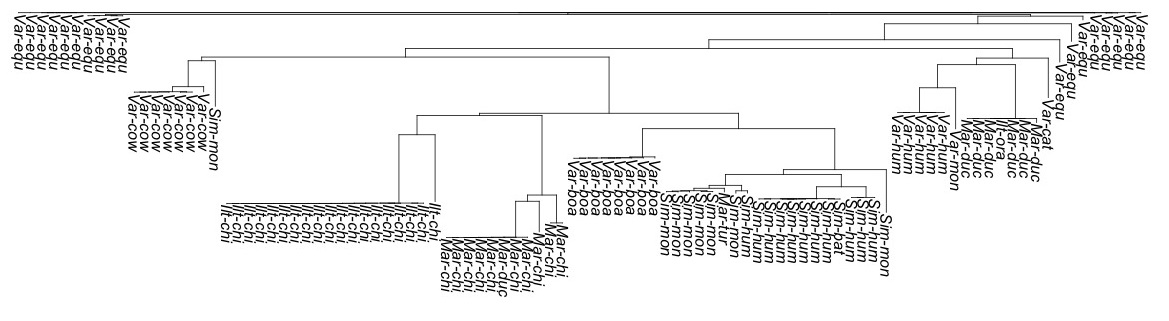}
  \caption{Phylogenetic tree generated for {\it Alphaherpesviridae}. It is broadly similar to the dendrograms
        in Figs.  \ref{Herphcal} and \ref{Herpen}, but differs in some details. }
  \label{Phylo_Herp1}
 \end{figure}

 The same cannot be said for the phylogenetic tree generated by NJ as before and shown in Fig.
 \ref{Phylo_Herp1}.  For instance, {\it Varicellovirus} from host equidae are partitioned into two clusters early.
 Also, {\it Mardivirus} from chicken and duck hosts are not cleanly separated.  On the other hand, most other
 population-host pairs are fairly well separated.  Overall, Fig. \ref{Phylo_Herp1}
  does not give as good a clustering as the panels in Figs. \ref{Herphcal}  and \ref{Herpen}.


\section{Conclusions}
\label{conclude}

In this paper we have presented a method for clustering categorical data in low, high, and varying dimensions.
We began with relatively small dimensions, up to 35 for the 
${\sf SOYBEAN}$ data, and studied the way our method seemed to
improve over other methods.  Specifically, we ensembled over dissimilarity
matrices in an effort to represent the distance between data points more
accurately.  Our theoretical work in Sec. \ref{whyensemble} provides
some formal justification for why this sort of technique should perform
well in some generality.

Then we turned to the clustering of high dimensional categorical data,
focusing on genomic data.  We extended our ensemble method for
low dimensional data to high dimensional categorical vectors of equal length
by adding a layer of ensembling:  We obtained dissimilarity matrices 
by ensembling over randomly selected dimensions.  We then used our method on
categorical vectors of different lengths by artificially making them the
same length through alignment procedures.  Again, our ensembling
method performed better than the other methods we tested.
In particular, we compared the output of our method in this case
to phylogenetic trees.  While not strictly scientific, the dendrograms
we generated can be interpreted physically and differ in some
important respects from phylogenetic trees generated from the same data.

Throughout we have used a large number of simulated and real data examples
to buttress the intuition behind the technique and formal results.
We comment that there are many other tests of the general methodology
that could be done.  For instance, in our clustering of viral genomes
we could have included incomplete genomes.  However, it many cases the
incomplete genomes had over 90\% of the nucleotides missing and we
thought this insufficient for good conclusions.  

\section*{Acknowledgments}
\vspace{-.2cm}
The authors gratefully acknowledge research support from NSF-DTRA grant DMS-1120404. Authors would like to thank Daniela Witten and Liang Bai for providing  the codes of sparce clustering and MWKM, and Mehdi R.M. Bidokhti for his comments of Virology part. 

\bibliographystyle{chicago}

\end{document}